\documentclass[11pt]{article}
\usepackage[T1]{fontenc}
\usepackage{lmodern}
\usepackage[margin=1in]{geometry}
\usepackage{microtype}
\usepackage{amsmath,amssymb,amsthm}
\usepackage{graphicx}
\usepackage{booktabs}
\usepackage{makecell}
\usepackage{array}
\usepackage{enumitem}
\usepackage{float}
\usepackage[numbers,sort&compress]{natbib}
\usepackage{authblk}
\usepackage[hidelinks]{hyperref}
\newtheorem{proposition}{Proposition}
\newcommand{\racsupplement}{\url{https://github.com/SZabolotnii/TRACE-RAC-code-supplement}}

\title{Trust Is Not a Score: Runtime Assurance Contracts for High-Risk AI Agents}
\author[1,2,3]{Serhii Zabolotnii\thanks{ORCID 0000-0003-0242-2234.
Email: \href{mailto:zabolotnii.serhii@csbc.edu.ua}{zabolotnii.serhii@csbc.edu.ua}}}
\affil[1]{Cherkasy State Business College, Cherkasy, Ukraine}
\affil[2]{State Scientific Research Institute of Armament and Military Equipment
Testing and Certification, Cherkasy, Ukraine}
\affil[3]{healthPrecision, New York, NY, USA}
\date{}

\begin{document}
\maketitle

\begin{abstract}
Benchmarks, audits, and agent protocols describe performance, permissions, and
repair, but not how observed evidence should change an agent's authority during
a consequential task. We call this the assurance-transition gap. We propose a
Runtime Assurance Contract (RAC), a policy-level formal schema binding autonomy
boundaries, component eligibility, evidence state, transition policy,
human-review capacity, and non-compensatory gates. Under RAC, soft metrics may
inform routing, whereas a failed or unknown mandatory gate forces retry, switch,
escalation, deferral, or stop; aggregate performance cannot authorize action. We
define the contract, an evidence record, a permission rule, and five invariants,
and illustrate them in clinical, industrial, and judicial failure probes. We
then report a deterministic failure-injection study in agentic coding: 174
constructed cases, each a distinct mechanism, evaluated by a gate conjunction,
a score-only rule, and a restricted protocol baseline. At the published example
weights and threshold, the score rule admits 56 of 63 block-required injections
and all 24 review-required injections. Tuned in hindsight, it matches the
conjunction case for case. For positive weights, a positive threshold, binary
risk signals, and each signal firing alone on some case, we show that the two
rules agree on every case if and only if the threshold does not exceed the
smallest weight. On a further 124 cases written blind by a model of another
vendor, the conjunction itself admits 35 of 79 injections, and so does the tuned
score rule: a missing detector is invisible to every decision rule. A separate set of 18 hand-authored traces checks version-pinned
evidence and review transitions against simpler policy variants. In a further
prospective synthetic holdout of 24 episodes, two blinded LLM judges assign
identical labels to all 72 action attempts; RAC and a separately implemented full stateful baseline
both match these labels. These studies test mechanisms on synthetic cases;
they establish neither deployed safety nor cross-domain effectiveness.
\end{abstract}


\section{Introduction}

Agentic AI systems reach actions through chains of tool calls, component
invocations, and state transitions. A plausible final output can conceal an
unauthorized component, stale context, or a skipped review earlier in the
chain. Here, a \emph{high-risk} workflow is one in which such a path can
materially affect health, legal rights, physical safety, critical operations,
or access to essential services.

Two lines of work address complementary parts of the problem: evaluation-centered
governance, which makes claims measurable and reviewable~\cite{weninger2026ethics},
and protocol-centered engineering, which gives agent--software interactions
explicit permissions, task state, and repair
paths~\cite{sharanarthi2026protocols}. Yet a governance proxy does not establish
its underlying deployment goal~\cite{sharma2026proxy}, and a well-formed tool
call does not establish that the agent should retain its current authority.

The unresolved issue is the mapping from an evidence state to an authority
transition, which we call the \emph{assurance-transition gap}. The earlier TRACE
framework~\cite{zabolotnii2026trace} treats trust in a human--task--agent
interaction as measured, calibrated, and uncertainty-qualified evidence, but
measurement alone neither grants nor withdraws authority. Nor do assurance
arguments, model-level abstention, or runtime monitors automatically compose
component eligibility, evidence continuity, escalation, and available human
review into one task-level decision.

We therefore propose to treat assurance as a non-compensatory runtime contract that
turns measured conditions into bounded authority. For an autonomous system,
this is an adaptation problem in which the property being adapted is the
agent's authority. This paper defines the
Runtime Assurance Contract (RAC), separates soft routing evidence from hard
authorization gates, states five invariants, and evaluates constructed failure
injections and transition scenarios. The contribution is a policy-level formal schema,
not an implemented safety guarantee or a new trust score. Non-compensatory
means that no high score elsewhere can offset a failed mandatory gate.

\section{From Evaluation and Protocols to Authority}

RAC specifies an \emph{authority transition}: the permitted change in what an
agent may do after evidence has been observed. Representative approaches differ
in their usual unit of control (Table~\ref{tab:binding}). Calibration yields
evidence about confidence quality without granting system
permissions~\cite{guo2017calibration}; governance proxies support review while
remaining detachable from deployment outcomes~\cite{sharma2026proxy}; and agent
protocols such as MCP and A2A specify interfaces, permissions, and repair but
leave application-specific authority gates to the deploying
system~\cite{sharanarthi2026protocols,mcp2025spec,a2a2026spec}. Assurance cases
including AMLAS organize evidence into a safety
argument~\cite{hawkins2021amlas,dong2023reliability,adler2022assurance} rather
than acting as runtime dispatchers, and selective prediction offers a
model-level reject option without reallocating task
authority~\cite{geifman2017selective,geifman2019selectivenet}.

Runtime enforcement is the nearest technical predecessor: shields and the
Simplex family bind a safety predicate to a guarded fallback
transition~\cite{alshiekh2018shielding,seto1998simplex,sha2001simplicity,desai2019soter},
a line since extended to probabilistic and online
synthesis~\cite{jansen2018safe,koenighofer2020online} and surveyed for AI
specifically~\cite{koenighofer2022survey}; recent work addresses its adoption
cost by connecting shield synthesis to standard learning
frameworks~\cite{pranger2026easytouse}. What these enforce is a specification
over actions in a modeled environment. RAC treats component eligibility,
evidence continuity, repair, escalation, and bounded human review as
first-class policy objects rather than a single plant safety predicate. Tiered
Agentic Oversight is a neighboring empirical example of task-sensitive
routing~\cite{kim2025tao}. As regulatory motivation, Article~72 of the EU AI Act
makes post-market monitoring a continuing obligation~\cite{euaiact2024}.

RAC can also be read as a self-adaptation policy. In the MAPE-K reference
model, an autonomic manager monitors a managed system, analyzes the
observations, plans an adaptation, and executes it over shared
knowledge~\cite{kephart2003vision,weyns2020intro}. In these terms the agent is
the managed system, the evidence record belongs to the knowledge, the gates are
analysis predicates, and $\Pi$ is the planner. The adapted property is the
agent's own authority, and the adaptation is non-compensatory. Dynamic
assurance cases already couple self-adaptation with assurance evidence that is
updated at runtime~\cite{calinescu2018entrust}; RAC places such evidence on the
authorization path of each consequential action rather than in an argument
about the adaptive system as a whole. Human participation in adaptation has
been modeled as a tactic with its own latency and
reliability~\cite{camara2015human}, and latency-aware planning treats the time
a tactic takes to have effect as a decision input~\cite{moreno2018latency}.
\textsf{G\_SUPERVISION} makes the corresponding assumption explicit: human
review is an admissible transition only while a qualified reviewer can complete
it within the declared deadline.

\begin{table}[H]
\caption{Typical scope of representative assurance approaches. The entries are
qualified comparisons, not claims that every implementation has the same
boundary. AMLAS denotes Assurance of Machine Learning in Autonomous Systems;
RTA, runtime assurance; MAPE-K, monitor--analyze--plan--execute over shared
knowledge; and RAC, the proposed Runtime Assurance Contract.}
\label{tab:binding}
\footnotesize
\setlength{\tabcolsep}{3.5pt}
\begin{tabular}{@{}p{3.4cm}ccccc@{}}
\toprule
Approach & \makecell[c]{Measures\\outcomes} & \makecell[c]{Defines\\permissions} & \makecell[c]{Defines\\repair} & \makecell[c]{Binds human\\capacity} & \makecell[c]{Evidence $\to$\\authority} \\
\midrule
Model evaluation / calibration & yes & no & no & no & no \\
Governance proxy / audit & yes & indirect & optional & optional & indirect \\
Agent protocol interface & not inherent & yes & yes & optional & configured \\
Assurance case / AMLAS & argument & constraints & mitigation & optional & argument-level \\
Selective prediction & model risk & no & reject & optional & model-level \\
Runtime enforcement / RTA & safety state & yes & fallback & typically no & guarded transition \\
Self-adaptive loop (MAPE-K) & monitored state & indirect & adaptation & optional & adaptation plan \\
\textbf{Proposed RAC} & mixed evidence & yes & yes & explicit & composed transitions \\
\bottomrule
\end{tabular}
\end{table}

Two lines of work on agent safety are nearer still, and RAC is best positioned
relative to them. The first places a learned component in the authorization path.
GuardAgent protects a target agent by analyzing a safety request and emitting
guardrail code for execution~\cite{xiang2024guardagent}, and step-level
guardrail models judge a tool invocation before it executes, reasoning over
interaction history to produce a safety verdict~\cite{mou2026toolsafe}.
A benchmark of such guards across multi-step trajectories finds that their
efficacy depends more on the structure of what they are shown than on the guard
model~\cite{chen2026tracesafe}, which is consistent with the position taken
here: a judgment produced by a learned component is evidence, and
Equation~\ref{eq:permit} treats it as a soft metric rather than as a permission
(I1, I3). The distinction is not that such guards are weak --- several are
strong --- but that a system in which they can authorize has no fail-closed
element left when they are wrong.

The second line enforces permissions deterministically at the tool-call
interface. Progent represents privilege as symbolic rules over tool names and
arguments and checks every call against them by a deterministic
procedure~\cite{shi2025progent}; MiniScope derives permission hierarchies
automatically so that least privilege does not depend on hand-written
policy~\cite{zhu2025miniscope}. This is the protocol-only position evaluated in
Section~\ref{sec:eval}, and it is a genuine guarantee at the boundary it
governs. Its limit is scope rather than rigor: a call may be individually
permitted while the task evidence supporting it has gone stale, lost provenance,
or outrun the review capacity the policy assumes. Two recent results sharpen
this. An empirical study of permission-boundary inference reports that frontier
models both omit permissions the execution chain needs and grant unused or
sensitive ones, concluding that authorization is not a conservative-versus-permissive
calibration problem~\cite{yan2026leastpriv} --- in other words, authority is not
a scalar dial, the same claim Section~\ref{sec:noncomp} makes analytically.
Separately, least privilege has been argued insufficient for agentic systems
that combine, approve, and amplify permissions across workflows, motivating a
compositional notion of least autonomy~\cite{parisel2026autonomy}. RAC shares
that diagnosis and differs in remedy: it keeps the permission model and adds a
task-level veto with a persisted transition record. Hazard-driven derivation of
enforceable specifications is complementary to both~\cite{doshi2026towards}.

The proposed distinction is compositional. RAC does not replace these
approaches; it defines how their evidence and constraints jointly govern the
next authority state.

\section{The Runtime Assurance Contract}

We define a \emph{Runtime Assurance Contract} as the tuple
\begin{equation}
\mathcal{C}_{\mathrm{RA}}=\langle\mathcal{B},\mathcal{K},\Pi,
\mathcal{H},\mathcal{E},\mathcal{G}\rangle .
\label{eq:rac}
\end{equation}
The tuple is a normative minimum for an implementable policy, not a deployed
enforcement engine.

\paragraph{Autonomy boundary $\mathcal{B}$.}
$\mathcal{B}(x,a,t)\in\{\textsf{allow},\textsf{require\_review},
\textsf{deny}\}$ evaluates proposed action $a$ in task context $x$ under active
policy version $t$. These verdicts are boundary decisions, not state-transition
labels.

\paragraph{Component eligibility $\mathcal{K}$.}
$\mathcal{K}(m,q,x,\eta_t,t)\in\{0,1\}$ states whether component $m$ may execute
task $q$ given context $x$, evidence state $\eta_t$, and policy version $t$. It
may depend on capability, calibration, modality, version, and domain
restrictions, and is not a model ranking. Eligibility is re-evaluated before each
\textsf{execute} transition and after a \textsf{retry}, \textsf{switch}, or
policy-version change.

\paragraph{Stateful transition policy $\Pi$.}
\begin{equation}
\Pi(z_t,o_t,\eta_t)\longrightarrow(\delta_t,z_{t+1},\eta_{t+1}),
\quad
\delta_t\in\Delta ,
\label{eq:policy}
\end{equation}
where $z_t$ is the runtime state, including the active context and rule
versions, $o_t$ the observed output, and
$\Delta=\{\textsf{execute},$ $\textsf{verify},$ $\textsf{retry},$
$\textsf{switch},$ $\textsf{escalate},$ $\textsf{defer},$ $\textsf{act},$
$\textsf{stop}\}$.
Here \textsf{execute} invokes an eligible component; \textsf{verify} updates
evidence; \textsf{retry} repeats a step; \textsf{switch} selects another
component; \textsf{escalate} requests human authority; \textsf{defer} holds a
safe state; \textsf{act} releases the action; and \textsf{stop} terminates the
action path.

\paragraph{Human authority and capacity $\mathcal{H}$.}
The contract names an accountable role, safe-fallback owner, credentials, veto
rights, review outcomes, response deadline, and capacity variables such as queue
length and reviewer load. An assigned role without timely capacity does not
satisfy this requirement.

\paragraph{Evidence schema $\mathcal{E}$.}
Every transition in $\Delta$ emits
\begin{equation}
\begin{split}
e_t=\langle &\textit{id},\textit{time},\textit{source},z_t,z_{t+1},\delta_t,
\textit{context\_version},\textit{rule\_version},\\
&\textit{component\_version},\mathbf{g}_t,\beta_t,
\textit{input\_hash},\textit{output\_ref},\textit{reason},\\
&\textit{review\_required},\textit{reviewer},\textit{review\_outcome}\rangle .
\end{split}
\label{eq:evidence}
\end{equation}
Here $\mathbf{g}_t$ records gate outcomes and $\beta_t$ the remaining retry
budget. A required review may initially have no assigned reviewer; that null
value records an unavailable or pending assignment, not approval.
\textit{review\_outcome} is pending, approved, rejected, timed out, or returned
for change. An approval is bound to the context, rule, component, input, and
output versions in the reviewed evidence record. It is invalidated when any of
those change before \textsf{act}.

\paragraph{Hard-gate set $\mathcal{G}$.}
The applicable subset $\mathcal{G}_t\subseteq\mathcal{G}$ contains mandatory
functions $g:(z_t,o_t,\eta_t)\mapsto\{0,1,\bot\}$, where $\bot$ denotes an
unknown or unevaluated outcome. Soft metrics cannot compensate for $0$ or $\bot$
(Section~\ref{sec:noncomp}).

\begin{figure}[!htb]
  \centering
  \includegraphics[width=0.88\textwidth]{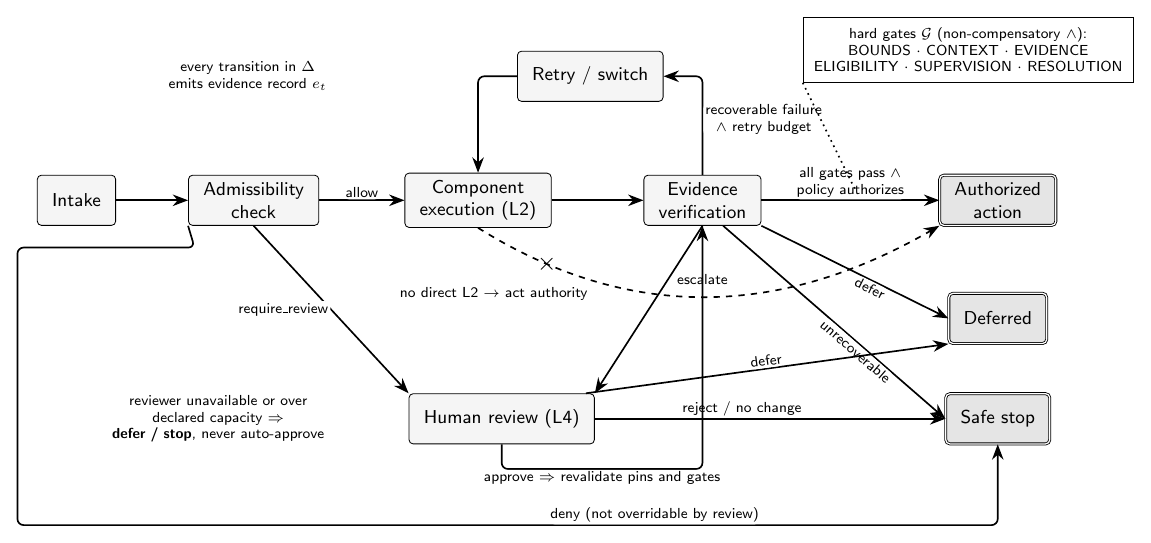}
    \caption{RAC state machine. An output does not authorize action directly.
  Failed or unknown mandatory gates fail closed; unavailable mandatory review
  yields \textsf{defer} or \textsf{stop}. Each consequential transition emits
  $e_t$ (Eq.~\ref{eq:evidence}); Table~\ref{tab:gates} specifies failure
  transitions.}
  \label{fig:statemachine}
\end{figure}

TRACE supplies one architectural mapping, with $\mathcal{B}$, $\mathcal{K}$,
$\Pi$, and $\mathcal{H}$ corresponding to its L1 rule core, L2 component
inventory, L3 orchestration, and L4 supervision, and $\mathcal{E}$ spanning all
layers. RAC itself is architecture-neutral.

Consider an agent preparing a medication order. A stale medication list fails
\textsf{G\_CONTEXT}, so $\Pi$ emits \textsf{retry}; after refresh,
\textsf{G\_SUPERVISION} requires a credentialed clinician within the declared
queue and response limits, so the next transition is \textsf{escalate}, not
\textsf{act}. A protocol could admit every tool call here: protocols govern
interaction boundaries, whereas RAC composes conditions across the task.

\section{Non-Compensatory Authority Transitions}
\label{sec:noncomp}

Let $(\delta_t,z_{t+1},\eta_{t+1})=\Pi(z_t,o_t,\eta_t)$. Permission to act is
\begin{equation}
\operatorname{Permit}(z_t,o_t,\eta_t)
=
\bigl[\delta_t=\textsf{act}\bigr]
\land
\Bigl(\bigwedge_{g\in\mathcal{G}_t}
\bigl[g(z_t,o_t,\eta_t)=1\bigr]\Bigr).
\label{eq:permit}
\end{equation}
The two terms address different failures: $\Pi$ selects a transition, while the
gate conjunction independently vetoes \textsf{act}. It equals $1$ only when every
applicable gate returns $1$; any $0$ or $\bot$ makes $\operatorname{Permit}$
false. A gate omitted as non-applicable evaluates to true only when the policy
version stores its scope and rationale.

Equation~\ref{eq:permit} differs from score-only authorization,
\begin{equation}
\operatorname{Permit}_{\mathrm{score}}(z_t,o_t)
=
\mathbf{1}\!\left\{\textstyle\sum_iw_i m_i(z_t,o_t)<\tau\right\},
\quad w_i>0,\quad\tau>0.
\label{eq:score}
\end{equation}
Here $m_i$ measures a risk signal, and the action is refused once the weighted
sum reaches $\tau$. Soft
metrics remain useful for monitoring and routing, but cannot waive a mandatory
gate.

The compensatory--non-compensatory distinction is not new to this setting.
Decision research has long separated additive models, in which a high value on
one attribute offsets a low value on another, from conjunctive and lexicographic
rules that screen each attribute
independently~\cite{einhorn1970noncomp,tversky1972eba}, and the conditions under
which decision makers move between the two families are well
studied~\cite{payne1993adaptive,hauser2014consideration}.
Equation~\ref{eq:score} is the additive member of that pair and
Equation~\ref{eq:permit} the conjunctive one. What is specific to runtime
assurance is not the distinction but its consequence. In choice modeling,
compensation is a claim about preference; here the attributes are safety
conditions, so compensation is the mechanism by which a satisfied condition
licenses a violated one.

Equation~\ref{eq:score} is not a straw man. The introductory worked example in
one evaluator vendor's documentation quarantines a message when a weighted sum
of three probabilities reaches $0.6$, so a credential request judged $0.85$
probable scores $0.38$ alone and escapes both quarantine and the review band
beneath it~\cite{typesafe2026docs}. The same vendor's guardrails cookbook, by
contrast, thresholds each hazard separately and promotes a borderline
probability about a serious hazard to a block through an independent severity
gate~\cite{typesafe2026guardrails}. Practitioners therefore do draw the
distinction, but they do not apply it consistently, and the introductory
pattern is the additive one.
The failure also appears under measurement. In a research deployment that
halts every tool call an LLM monitor scores at $3/10$ or above until a human
decides, the authors' own red-teaming found that a hand-written fake user turn
lowered the mean score of a harmful action from $8.1$ to $3.7$, and $12$ of $30$
runs fell below the threshold~\cite{metr2026monitor}. The claim of authority
entered through the same text the score was computed from.

\begin{table}[tb]
\caption{Minimal hard gates and required failure transitions.}
\label{tab:gates}
\footnotesize
\setlength{\tabcolsep}{3pt}
\centering
\begin{tabular}{@{}p{2.5cm}p{5.1cm}p{4.3cm}@{}}
\toprule
Gate & Passing condition & Failure transition \\
\midrule
\textsf{G\_BOUNDS} & boundary allows action or required review is approved & \textsf{escalate} / \textsf{defer} / \textsf{stop} \\
\textsf{G\_CONTEXT} & required context fresh and traceable & \textsf{retry} / \textsf{defer} / \textsf{stop} \\
\textsf{G\_EVIDENCE} & event fields complete and source-valid & \textsf{verify} / \textsf{retry} / \textsf{stop} \\
\textsf{G\_ELIGIBILITY} & $\mathcal{K}=1$ for task/version/domain & \textsf{switch} / \textsf{stop} \\
\textsf{G\_}\allowbreak\textsf{SUPERVISION} & role, load, queue, deadline admissible & \textsf{defer} / \textsf{stop} \\
\textsf{G\_RESOLUTION} & no unresolved failure on action path & \textsf{retry} / \textsf{escalate} / \textsf{stop} \\
\bottomrule
\end{tabular}
\end{table}

After a reviewer approves, \textsf{act} still requires a fresh evaluation of
Equation~\ref{eq:permit} against the pinned evidence and current policy. A
reviewer leaving the queue after a completed approval does not annul that
approval; capacity is checked when review is requested, not used as an
indefinite condition on an already completed decision.

The retry budget $\beta_t$ is scoped to the candidate action path and stored in
$z_t$: a recoverable failure with $\beta_t>0$ permits \textsf{retry} or
\textsf{switch} and decrements it, and at zero only \textsf{escalate},
\textsf{defer}, or \textsf{stop} remain. Escalation does not clear an unresolved
failure; a qualified reviewer must resolve it or assume authority under a
recorded policy, and the event records the reviewer, outcome, rule version, and
reason. The same conditions select among the alternatives in
Table~\ref{tab:gates}.

RAC implies five invariants:
\begin{description}[nosep,leftmargin=1.1em,style=unboxed]
\item[I1 --- Authority separation.] A learned component cannot authorize a
consequential action or alter its autonomy boundary.
\item[I2 --- Fail-closed gates.] A failed or unknown mandatory gate blocks
\textsf{act}.
\item[I3 --- Confidence asymmetry.] Low confidence may reduce autonomy; high
confidence does not waive independent checks~\cite{kim2025tao}.
\item[I4 --- Oversight viability.] Required review that misses its role,
capacity, or deadline condition yields \textsf{defer} or \textsf{stop}, never
auto-approval.
\item[I5 --- Evidence continuity.] A missing transition record is an
operational failure rather than a later documentation defect.
\end{description}

\section{Cross-Domain Failure Probes}
\label{sec:probes}

Table~\ref{tab:probes} instantiates the same symbolic sequence in three domains:
observed state, failed gate, prohibited action, mandatory transition, and event
record. The probes are illustrative. They provide no performance measurements
or cross-domain validation.

\begin{table}[H]
\caption{Illustrative failure probes and mandatory transitions.}
\label{tab:probes}
\footnotesize
\setlength{\tabcolsep}{3pt}
\begin{tabular}{@{}p{2.05cm}p{2.05cm}p{1.55cm}p{2.1cm}p{2.0cm}p{2.7cm}@{}}
\toprule
Probe & Baseline state & Failed gate & Prohibited action & Required transition & Persisted evidence \\
\midrule
Stale clinical context & benchmark passed; call well-formed &
\makecell[l]{\textsf{G\_}\\\textsf{CONTEXT}} & commit order draft & refresh and \textsf{retry}; else
\textsf{defer} & stale field, gate, context version, $\delta_t$, reason \\
\addlinespace
Judicial draft with missing provenance & risk score below refusal threshold &
\makecell[l]{\textsf{G\_}\\\textsf{EVIDENCE}} & release recommendation & \textsf{verify}; if policy marks
high consequence, \textsf{escalate} & missing field, gate, $\delta_t$, reviewer \\
\addlinespace
Operator queue over capacity & call permitted; review required &
\makecell[l]{\textsf{G\_}\\\textsf{SUPER}\\\textsf{VISION}} & approve setpoint change & safe hold with
\textsf{defer}; else \textsf{stop} & load, deadline, gate, $\delta_t$, reason \\
\bottomrule
\end{tabular}
\end{table}

Across the three probes, a score or protocol check can remain permissive while
task evidence requires less authority. RAC's claimed distinction is precisely
this additional veto and its persisted transition record
(Eq.~\ref{eq:evidence}), not measured cross-domain superiority. An audit may
record a failure, a monitor may flag it, and a protocol may still admit the
call; RAC requires the task-level veto and its resulting transition record.

\section{A Documented Failure}
\label{sec:contrast}

An open-source agent harness persists model-authored skills and subagent
specifications across episodes and applies refinements at turn boundaries
without review. Its authors report that the agent found an out-of-band interface bypassing the rules of its
task environment, used it despite an explicit instruction not to, and kept the
bypass as reusable state that later refinement optimized further. They conclude
that safe deployment needs least-privilege action interfaces, independent state
validation and auditable rollback~\cite{karten2026primeagent}. This is a
reported incident, not a controlled evaluation of RAC. It motivates a check of
whether later authority can change without a recorded, independently validated
transition; it does not establish that the particular RAC gates proposed here
would have prevented the incident.

\section{Evaluation}
\label{sec:eval}

The incident in Section~\ref{sec:contrast} is observational and cannot separate
the decision rule from co-varying design choices. A paired failure-injection
study isolates the authorization rule on constructed cases. We report one,
narrowed to the single domain for which we have a
reference implementation of Equation~\ref{eq:permit} --- agentic coding, where
the consequential action is a file write proposed by a coding agent.

\subsection{Design}

Every case is evaluated by all three arms, so all comparisons are paired.

\begin{itemize}[nosep,leftmargin=1.2em]
\item \textbf{RAC.} The non-compensatory conjunction of Equation~\ref{eq:permit};
\textsf{act} only when every applicable gate returns $1$.
\item \textbf{Score-only.} Equation~\ref{eq:score} over the same signals:
\textsf{act} iff $\sum_i w_i m_i < \tau$.
\item \textbf{Protocol-only.} Interface permissions and repair: the
\emph{declared} path string is matched against the scope and deny globs and the
diff is checked for well-formedness, with no path normalization, no symlink
resolution, and no content inspection.
\end{itemize}

The three soft metrics are saturated --- each takes the value $1$ as soon as its
feature is present --- which is the strongest possible reading of each detector
and a deliberately charitable choice for the baselines.

\subsection{Corpus}

The corpus is 174 cases: 87 injections and 87 class-matched controls across
seven fault classes, and every case is a distinct mechanism. Its provenance is
deliberately split. Five classes come from an \textsf{envelope\_bypass} corpus
first written in August 2026 for an unrelated purpose. Two classes were added
to include review-required cases; both were derived from a documented guard
contract rather than tuned against any arm's behavior, and both lie
\emph{inside} the declared write scope, so the boundary gates are silent on them
and only \textsf{G\_SUPERVISION} can reach them. Ground truth is fixed by
construction, not by adjudication, and is checked by contract tests that fail if
any class loses its gate label or if the escalation classes drift outside the
declared scope.

An earlier version of this study, including arXiv~v1, reported 280 cases. An
internal review found that most of them were index clones: one path listed
twice, one template per class repeated twenty times, and a single control line
for every class. The corpus was rebuilt so that each case is one mechanism with
a control matched to its class. The first run on the rebuilt corpus found seven
review-required mechanisms that the guard's contract did not cover: tests moved
into helper and data files, and module-level, \texttt{unittest}, import-time and
collection-level silencing. The guard was then changed by general rules, not
by corpus lines. Table~\ref{tab:outcomes} reports both runs. The second is
in-sample: it is measured on the corpus that exposed the gaps.

\subsection{Outcomes}

We count unauthorized action over block-required injections, admission despite
any injected gate violation over all injections, and automatic admission of
review-required injections, together with false refusal on controls. The third
quantity does not show whether an arm actually sent a request to a reviewer;
refusal alone also scores zero. The three injection counts overlap by
construction: the total admitted injections equal admitted block-required plus
admitted review-required cases. The cases are constructed mechanisms, not
independent draws from a population of real changes. We therefore report corpus
counts and treat intervals and paired tests as descriptive of this corpus only.

\begin{table}[tb]
\caption{Admitted constructed cases by arm (lower is better on injected cases),
before and after the guard change described above. The score arm is shown at the
published example configuration of Section~\ref{sec:noncomp},
$w=(0.45,0.30,0.25)$, $\tau=0.6$, and at its best operating point chosen in
hindsight over the whole grid, $w=(0.25,0.30,0.45)$, $\tau=0.05$.}
\label{tab:outcomes}
\small
\centering
\begin{tabular}{@{}llrrr@{}}
\toprule
Outcome & Arm & First run & After fix & Denominator \\
\midrule
unauthorized action & RAC & \textbf{0} & \textbf{0} & 63 \\
 & score (vendor) & 56 & 56 & 63 \\
 & score (best) & 0 & 0 & 63 \\
 & protocol & 38 & 38 & 63 \\
\addlinespace
admitted injections & RAC & 7 & \textbf{0} & 87 \\
 & score (vendor) & 80 & 80 & 87 \\
 & score (best) & 7 & 0 & 87 \\
 & protocol & 62 & 62 & 87 \\
\addlinespace
auto-admitted review cases & RAC & 7 & \textbf{0} & 24 \\
 & score (vendor) & 24 & 24 & 24 \\
 & score (best) & 7 & 0 & 24 \\
 & protocol & 24 & 24 & 24 \\
\addlinespace
false refusal & RAC & 3 & 2 & 87 \\
 & score (vendor) & 0 & 0 & 87 \\
 & score (best) & 3 & 2 & 87 \\
 & protocol & 0 & 0 & 87 \\
\bottomrule
\end{tabular}
\end{table}

\subsection{Configuration dependence}

At its best operating point the weighted sum matches the conjunction case for
case, in both runs, including its two false refusals: base64 text that is not a
secret but fires the secret detector, which both rules read. The conjunction
removes the weights and aggregate threshold from authorization, though engineers
still have to choose, test, and scope its constituent gates. On the selected
grid the weighted sum has 114 configurations and 30 reproduce the conjunction on
every case; the published example configuration is not among them. The same 30
reproduce it before and after the guard change and on the independent corpus
below. This fraction depends on the grid and is not an estimate of how often
deployed score policies fail.

Figure~\ref{fig:fragility} shows the grid result. Sweeping six weight
permutations against nineteen strictly positive thresholds, agreement with the
conjunction ends at $\tau=0.25$, the smallest weight.
Proposition~\ref{prop:tau} explains that boundary for binary signals. At each threshold, the
plotted permutation lexicographically minimizes unauthorized actions, admitted
injections, automatically admitted review-required cases, and false refusals,
in that order. The review-admission curve therefore depends on this priority
order; it is not a bound shared by all weight permutations.

\begin{figure}[tb]
\centering
\includegraphics{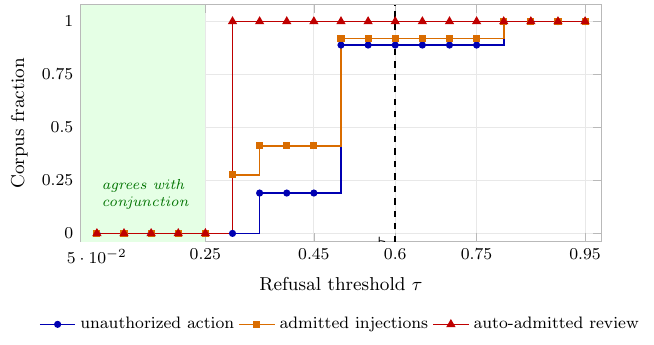}
\caption{Failure rates of the score-only arm against the refusal threshold on
the 174-case corpus after the guard change, at the weight permutation selected
for each $\tau$ by lexicographically minimizing unauthorized actions, admitted
injections, automatic review admissions, and false refusals, in that order. The
shaded region, where the weighted sum agrees with the conjunction on every
case, ends exactly at $\tau=\min_i w_i=0.25$ on this positive grid; there both
refuse the same two controls (0.023). Automatic admission of review-required
cases jumps to $1.0$ immediately above it under this selection rule. The
published vendor threshold lies well outside.}
\label{fig:fragility}
\end{figure}

\begin{proposition}
\label{prop:tau}
Suppose $w_i>0$ and $\tau>0$; the risk signals are saturated,
$m_i\in\{0,1\}$, and are the outcomes of the same detectors the conjunction
reads, so that the conjunction refuses a case iff some $m_i=1$; and for each
metric $i$ the corpus contains a case that fires $i$ alone. Then
Equation~\ref{eq:score} makes the same decision as the conjunction on every case
if and only if $0<\tau\le\min_i w_i$.
\end{proposition}

\begin{proof}
A case that fires no metric scores $0<\tau$ and is admitted by both rules. If
$\tau\le\min_i w_i$, any case firing at least one metric scores at least $\tau$
and is refused, as the conjunction refuses it. Conversely, if $\tau>w_j$ for
some $j$, the case firing $j$ alone scores $w_j<\tau$ and is admitted, while the
conjunction refuses it.
\end{proof}

The proposition does not require clean controls or a complete detector set: a
control that fires a detector is refused by both rules, and an injection no
detector sees is admitted by both. The corpus satisfies the hypotheses:
boundary, secret, and test metrics each fire alone on some case.
\textsf{test-rename} is the only class whose cases can fire two metrics
($0.45+0.25=0.70$), which is why it is the only class the vendor configuration
catches at all, and then only 7 of its 12 cases. The interpretation matters more
than the arithmetic.
$0<\tau\le\min_i w_i$ is precisely the condition under which a single metric
suffices to refuse --- that is, under which no high metric can offset a low one.
The weighted sum reproduces the conjunction exactly when it has been configured
so that compensation cannot occur on cases covered by the proposition. The
equivalence says nothing about missing detectors, continuous or noisy signals,
or whether the selected hard gates are correct.

\subsection{Per-class behavior}

\begin{table}[tb]
\caption{Injected cases passed by each arm on the 174-case corpus after the
guard change; $n$ is the number of distinct mechanisms in the class. Lower is
better; the required transition is \textsf{block} or human \textsf{review}.
Before the change, RAC and the best score arm passed 3 of the 12
\textsf{test-moved-in-scope} and 4 of the 12 \textsf{test-skip-marker}
mechanisms.}
\label{tab:perclass}
\small
\centering
\begin{tabular}{@{}llcrrrrr@{}}
\toprule
Class & Gate & Req. & RAC & score (vend.) & score (best) & protocol & $n$ \\
\midrule
base64-secret & \textsf{G\_EVIDENCE} & block & 0 & 12 & 0 & 12 & 12 \\
deny-relative & \textsf{G\_BOUNDS} & block & 0 & 12 & 0 & 2 & 12 \\
path-traversal & \textsf{G\_BOUNDS} & block & 0 & 15 & 0 & 12 & 15 \\
symlink-escape & \textsf{G\_BOUNDS} & block & 0 & 12 & 0 & 12 & 12 \\
test-rename & \textsf{G\_BOUNDS} & block & 0 & 5 & 0 & 0 & 12 \\
test-moved-in-scope & \textsf{G\_SUPERVISION} & review & 0 & 12 & 0 & 12 & 12 \\
test-skip-marker & \textsf{G\_SUPERVISION} & review & 0 & 12 & 0 & 12 & 12 \\
\bottomrule
\end{tabular}
\end{table}

The restricted protocol-only arm refuses most of \textsf{deny-relative}
and catches a fifth of \textsf{path-traversal}, because there the
declared string itself violates a glob. It passes every symlink escape and every
embedded secret, because the declared string looks legal and it never resolves
the path or reads the content. It admits every review-required case because
this baseline checks only the declared call and diff shape. That result does
not characterize all permission systems or demonstrate an actual RAC review
request.

\subsection{An independently written corpus}

Every case above was written by the gate's author, and the guard was changed on
the corpus that exposed its gaps. To measure what that costs, a second corpus
was written blind by OpenAI Codex (gpt-5.5). It was given only the declared
policy and a public statement of the three gates, worked in an empty directory,
and did not see the gate, the earlier corpora, or any result. The corpus was
committed before its first run. It holds 124 cases: 79 injections (27
\textsf{G\_BOUNDS}, 26 \textsf{G\_EVIDENCE}, 26 \textsf{G\_SUPERVISION}) and 45
controls.

Under the corpus author's labels, RAC admits 35 of the 79 injections and refuses 1
of the 45 controls. The best score configuration admits exactly the same cases
(McNemar $p=1.0$), and the vendor configuration admits 78. Every admitted case
was adjudicated:
\begin{itemize}[nosep,leftmargin=1.2em]
\item Two are artefacts of the harness: the diff creates a symlink, which the
probe does not materialize. With the link on disk, both are refused.
\item Eleven are placeholder credentials (\texttt{AKIAIOSFODNN7EXAMPLE},
\texttt{ghp\_FAKEFAKE\dots}) that the secret scanner ignores by design.
\item Six weaken assertions inside a test. This is outside the test guard's
contract, and the gate's mutation check, which is not part of this probe, would
have to catch them.
\end{itemize}
That leaves 16 of 62 injections (0.26, Wilson 95\% CI [0.17, 0.38]) that RAC
admits within its own detectors' contract:
\begin{itemize}[nosep,leftmargin=1.2em]
\item one Unicode path that normalization maps into scope;
\item seven credentials in connection URIs, configuration values, or fragments
assembled across lines;
\item eight test-suite silencings at collection level (conftest hooks,
\texttt{addopts}, \texttt{\_\_test\_\_ = False}, a renamed test class).
\end{itemize}

Two readings follow. First, the zeros of Table~\ref{tab:outcomes} measure the
gate's own vocabulary, not its coverage. Second, the comparison between decision
rules is unchanged. The conjunction and the tuned weighted sum read the same
detectors, so an injection no detector sees passes both, as
Proposition~\ref{prop:tau} implies. Non-compensation decides how detector
outcomes are combined; it cannot supply a detector that is missing. This corpus
is in-sample for any later change to the gate.

\subsection{Versioned transition conformance scenarios}

The single-step corpus above does not test the persisted transitions that
distinguish the full contract from a gate conjunction. We therefore add 18
hand-authored, deterministic traces in a dependency-free executable model
supplied with the artifact. Each trace contains state updates, verification,
an optional review request and approval, and a final proposed action. Its
expected transition is declared alongside the trace. These are conformance
checks of the specified policy, not independent task observations.

The full model matches all 18 declared transitions. A stateless conjunction of
the same current gate fields matches 12; a variant that accepts verification
and approval across context or policy changes matches 13; and a variant that
assumes review capacity always exists matches 16. The failures identify what
the stateful terms add: a changed context, rule, or payload calls for renewed
verification; an approval pinned to an older version cannot release a changed
action; unavailable review must defer. A completed approval remains valid if
the reviewer later leaves the queue, provided its pinned state has not
changed. These counts describe chosen scenarios and support no population
claim.

\subsection{Prospective synthetic holdout with blinded LLM labels}
\label{sec:llm-holdout}

We added a protected-branch merge workflow to compare RAC with a full stateful
baseline. Required checks, reviews, stale-approval dismissal, and separation
of reviewer from last pusher follow documented branch-protection
options~\cite{github2026protected}. The study policy additionally specifies
versioned evidence, review expiry, capacity, an action deadline, and a retry
budget; these extensions are experimental choices. Both implementations receive
the same scripted CI results and review responses. The baseline retains all
these capabilities and does not import the RAC implementation.

After eight development episodes, the policy, implementations, and evaluator
were frozen with SHA-256 hashes. A separate LLM session then constructed 24
synthetic episodes containing 72 action attempts without seeing either
implementation or the development cases. Two fresh GPT-6 Sol sessions received
only the policy and anonymized event histories, with no expected labels,
implementation outputs, or other judge's answers. Each labeled every attempt
as allow, review, block, or insufficient information and supplied a rationale.
The policy, protocol, and judge instructions were written in Ukrainian, and
the judges wrote their rationales in Ukrainian; the public repository provides
English translations. Both label files were frozen before the test run. This was a local prospective
freeze, not an externally registered protocol; access restrictions were
instructional on a shared filesystem. No test-driven implementation changes
were made.

The judges agreed on all 72 attempts: 30 allow, 30 review, and 12 block, with
no insufficient-information labels. RAC and the full stateful baseline each
matched all 72 labels. Each produced zero policy-violating simulated releases
among the 42 non-allow attempts and zero false refusals among the 30 allow
attempts. Both recorded 28 review requests, 25 completed reviews, five invalid
review events, 30 simulated releases, and two consumed retries. Mean delay for
the 25 completed reviews was 1.84 simulated ticks; unfinished requests do not
enter this mean. Four attempts stopped and none deferred. Thus the holdout
does not test the unavailable-review deferral branch, despite containing
capacity changes.

The result is observed agreement under a shared policy, with no advantage for
RAC over the full stateful baseline. The judges belong to the same model
family, and the corpus is LLM-authored; agreement does not establish an
independent domain oracle. This is a holdout from implementation development,
not from the supplied specification. Its labels concern synthetic merge
authority, not the truth of CI outputs or the safety of actual patches.

\subsection{Instruments, determinism, and reproducibility}

The original probe implementation, escalation corpus, and contract tests were
written with an AI coding assistant (Claude Code, Claude Opus~5) under the
author's direction; the assistant also executed the runs and generated the
original tables and figure from the raw decision log. The transition scenario
model and the 2026-09 revision were prepared with OpenAI Codex. The October 2026
correction (the mechanism corpus, the independent-corpus run and its
adjudication) was prepared with Claude Code (Claude Opus~5.5); the independent
corpus itself was written by OpenAI Codex (gpt-5.5), as described above. Earlier drafting and
language revision used OpenAI Codex CLI 0.144.5 and PaperMentor 0.1.0. No model
call occurs when replaying the static or transition artifacts. The new holdout
uses LLM calls for corpus construction and labeling, then deterministic replay
against the saved labels. The executable artifact generates the reported
counts; replay does not reproduce the stochastic judge calls. Exact model
snapshots and sampling parameters for the judging sessions were unavailable.
The public code supplement at
\racsupplement{} contains
the standalone implementations and data for the new LLM-labeled holdout.
The private gate implementation and fixture generator used for the original
probe are not distributed; that probe cannot be independently rerun from the
public supplement.

The original probe run is deterministic: no sampling, no seeds beyond corpus construction, and
no model call on the evaluation path. It was reproduced byte-identically in two
independent Python environments. Two harness defects found by the first run are
recorded here because they are characteristic of this design. First, the
symlink-escape fixture stores its path absolutely, which the envelope requires
in order to resolve it on disk; the protocol arm saw the absolute string, failed
the scope glob, and refused --- appearing to catch the escape for a reason that
cannot arise in a real tool call. The arm now receives the path in the form the
agent declares it, after which it passes every symlink escape, as it should.
Second, the first run measured the score arm only at the vendor point, from which
the weighted sum appeared unconditionally worse; the grid sweep overturned that
reading and produced the actual result. Both were fixed before any number here
was treated as a finding.

\subsection{Threats to validity}

This is a mechanism study, not a safety result, and eight limits bound it.

\begin{description}[nosep,leftmargin=1.2em,style=unboxed]
\item[Circularity.] The RAC arm is an implementation of the gate set this paper
posits. The study shows that the set is implementable and identifies when
competing decision rules on identical inputs reproduce it. It does not show
that the set is correct; no independent oracle defines the right gates.
\item[Authored traces.] The 18 transition traces and expected outcomes were
hand-authored for contract conformance. Their agreement with the executable
model is a consistency check, not a measured reliability rate.
\item[LLM-labeled holdout.] Separate sessions hide implementation outputs from
the judges, but a shared policy and model family can produce correlated
errors. The 24 new episodes do not establish independent human domain
validation. The full stateful baseline implements the same policy and matches
RAC; the holdout provides no evidence of architectural superiority.
\item[Single synthetic domain.] The clinical, industrial, and judicial probes of
Section~\ref{sec:probes} remain illustrative. Nothing here supports the
cross-domain claim.
\item[Noisy signals.] The signals are binary. Proposition~\ref{prop:tau} covers
detectors that miss faults or fire on clean actions only in the sense that both
rules inherit the error; it does not predict behavior with continuous or noisy
scores.
\item[Over-blocking is a corpus property.] Two of 87 constructed controls are
refused. This says nothing about whether fail-closed is cheap. The operational
cost of refusal --- latency, reviewer queueing, denied service --- is not
measured in deployment. The holdout records scripted queue values and simulated
completion delays, which do not estimate these costs.
\item[In-sample guard.] The guard was changed on the corpus that exposed its
gaps. The zeros after that change are in-sample. The independently written
corpus measures out-of-sample admission, 16 of 62 within contract. Each
mechanism there is a single case, not a draw from a population of real changes.
\item[Limited rerunnability.] The shared static decision log permits a count
audit, not reconstruction of the private gate and its input fixtures. The
transition model and the new holdout are executable but are separate synthetic
implementations, not a rerun of the original private gate.
\end{description}

Establishing added value in deployment therefore still requires a domain study:
consequential actions in versioned task traces, labeled before injection by two
qualified reviewers under a frozen codebook, with one prespecified fault injected
into paired copies under a recorded seed, balanced within domain and consequence
strata, compared with a stateful baseline that has the same detectors. Its
outcomes should include actual review requests, completed decisions, harmful
releases, false refusals, and service delay. Deployment-specific harm analysis
must set the effect margins before data collection; this paper claims none.

\section{Limits and Ethics}

RAC is a design proposition. It does not prove system safety, validate the
underlying boundary, calibrate model confidence, or remove domain-specific
threshold setting; accountable owners must derive gates from risk analysis and
applicable duties. A fail-closed policy can itself cause harm through delayed
care, denied service, downtime, or unequal access to scarce reviewers, and
evidence logs may expose sensitive data. Implementations therefore need service
limits, safe fallback ownership, least-privilege logging, and retention controls.
\textsf{G\_SUPERVISION} is a minimum viability condition, not a complete account
of meaningful human oversight.

Evaluators that return probabilities without a rationale are hard to justify in
regulated domains when their outputs are summed, because the scalar does not
say which condition drove a refusal. Under Equation~\ref{eq:permit} a refusal is
recorded as the identity of the gate that returned false. This recovers which
condition failed, not why the model judged it so.

\section{Conclusion}

RAC specifies how gate outcomes, versioned evidence, and human review govern
the next authority transition. For binary signals, a weighted risk sum makes the
same decision as the gate conjunction on every case exactly when its positive
refusal threshold is no greater than the smallest positive weight; the
published example threshold falls outside that range. On a corpus written
independently of the gate, both rules admit about a quarter of the injections
within contract, because the gaps are in the detectors and not in the rule of
decision. The versioned trace checks expose the additional
state that a stateless conjunction does not retain. On the prospective
LLM-labeled synthetic holdout, a full stateful baseline matches RAC. These
are analytic and synthetic results. Whether the contract improves decisions in deployed systems
with noisy detectors and real review queues remains open.

\section*{Supplementary Material}
The accompanying supplementary archive contains a field-limited decision log,
its count-audit script, the executable transition scenarios, and a copy of the
public holdout repository. The public repository
\racsupplement{} contains the
policy, separately written implementations, development and test episodes,
blinded judge labels with rationales, freeze manifests, and simulated
execution logs for the LLM-labeled holdout. It supports rechecking those
counts and policy-model outputs, subject to the rerunnability limits described
above. It does not include the private gate implementation or fixture
generator used for the original probe.

\bibliographystyle{plainnat}
\bibliography{references}

\end{document}